\pdfoutput=1

\documentclass[twoside]{article} %\usepackage{aistats2017}
\usepackage{natbib}
\usepackage{url}
\renewcommand\cite{\citep}

\usepackage{threeparttable}
\usepackage[accepted]{aistats2017}
\newcommand*{\affaddr}[1]{#1} % No op here. Customize it for different styles.
\newcommand*{\affmark}[1][*]{\textsuperscript{#1}}

\usepackage{amsmath}
\usepackage{etoolbox}
\usepackage{bm}
\usepackage{amsthm}
\usepackage{mathtools}
\usepackage{algorithm}
\usepackage{algorithmic}
\usepackage{subfigure}
\usepackage{xspace}
\usepackage{multirow}
\usepackage[capitalise]{cleveref}
\usepackage{array}
\usepackage{amsfonts} %mathbb

\usepackage{footnote}
\makesavenoteenv{tabular}
\makesavenoteenv{table}

\newcommand{\MyMapTemplatePrefixc}[4]{\expandafter#1\csname#3#4\endcsname{#2{#4}}} % it remembles a template: \#3#4 --> #2{#4}
\forcsvlist{\MyMapTemplatePrefixc {\def} {\mathcal}{mc}} {A,B,C,D,E,F,G,H,I,J,K,L,M,N,O,P,Q,R,S,T,U,V,W,X,Y,Z}  % e.g., \cA --> \mathcal{A}
\forcsvlist{\MyMapTemplatePrefixc {\def} {\mathbf}{b} } {A,B,C,D,E,F,G,H,I,J,K,L,M,N,O,P,Q,R,S,T,U,V,W,X,Y,Z}  %

\newcommand{\MyMapTemplatePrefixtb}[5]{\expandafter#1\csname#4#5\endcsname{#2{#3{#5}}}} % it remembles a template: \#3#4 --> #2{#4}
\forcsvlist{\MyMapTemplatePrefixtb {\def} {\tilde}{\mathbf}{t}} {A,B,C,D,E,F,G,H,I,J,K,L,M,N,O,P,Q,R,S,T,U,V,W,X,Y,Z}  % e.g., \tA --> \tilde{A}
\forcsvlist{\MyMapTemplatePrefixtb {\def} {\tilde}{\mathbf}{t}} {0,1,a,b,c,d,e,f,g,h,i,j,k,l,m,n,o,p,q,r,s,t,u,v,w,x,y,z}  % e.g., \ta --> \tilde{a}

\newcommand{\MyMapTemplateNoPrefix}[3]{\expandafter#1\csname#3\endcsname{#2{#3}}}
\newcommand{\itk}{_k}
\newcommand{\kp}{_{k+1}}
\newcommand{\km}{_{k-1}}

\def\ga{\alpha}
\def\gb{\beta}

\def\ge{\epsilon}

\def\gl{\lambda}

\def\gt{\theta}

\def\gD{\Delta}

\def\hga{\hat{\alpha}}
\def\hgb{\hat{\beta}}
\def\hgl{\hat{\lambda}}

\newcommand{\tabincell}[2]{\begin{tabular}{@{}#1@{}}#2\end{tabular}}

\def\bbR{{\mathbb R}}

\def\bbN{{\mathbb N}}

\def\st{\mbox{subject to~~}}

\newcommand{\inprod}[1]{\langle #1 \rangle}

\newtheorem{proposition}{Proposition}

\crefname{rl}{Rule}{Rule}
\newtheorem{condition}{Condition}
\crefname{condition}{Condition}{Condition}

\begin{document}

% If your paper is accepted and the title of your paper is very long,
% the style will print as headings an error message. Use the following
% command to supply a shorter title of your paper so that it can be
% used as headings.
%
%\runningtitle{Adaptive ADMM}

% If your paper is accepted and the number of authors is large, the
% style will print as headings an error message. Use the following
% command to supply a shorter version of the authors names so that
% they can be used as headings (for example, use only the surnames)
%
\runningauthor{Zheng Xu, Mario Figueiredo, Tom Goldstein}

\twocolumn[

\aistatstitle{Adaptive ADMM with Spectral Penalty Parameter Selection}

%\aistatsauthor{ Zheng Xu\affmark[1] \And M\'{a}rio A. T. Figueiredo\affmark[2] \And Tom Goldstein\affmark[1] }
%
%\aistatsaddress{ \affaddr{\affmark[1]University of Maryland} \And Unknown Institution 2 \And Unknown Institution 3 } 

\aistatsauthor{
Zheng Xu\affmark[1],\quad M\'{a}rio A. T. Figueiredo\affmark[2],\quad Tom Goldstein\affmark[1]\\
\affaddr{\affmark[1]Department of Computer Science, University of Maryland, College Park, MD}\\
\affaddr{\affmark[2]Instituto de Telecomunica\c{c}\~{o}es, Instituto Superior T\'{e}cnico, Universidade de Lisboa, Portugal}\\
%\email{\affmark[1]\{xuzh,tomg\}@cs.umd.edu, \affmark[2]mario.figueiredo@tecnico.ulisboa.pt}
}
\vspace{0.5cm}
]

\begin{abstract}
The \textit{alternating direction method of multipliers} (ADMM) is a versatile tool for solving a wide range of constrained optimization problems.
%, with  differentiable or non-differentiable objective functions.  
However, its performance is highly sensitive to a penalty parameter,  making ADMM often unreliable and hard to automate for a non-expert user.  We tackle this weakness of ADMM by proposing a method that adaptively tunes the penalty parameter to achieve fast convergence. The resulting {\it adaptive ADMM} (AADMM)  algorithm, inspired by the  successful Barzilai-Borwein spectral method for gradient descent, yields fast convergence and relative insensitivity to the initial stepsize  and problem scaling.
\end{abstract}

\section{Introduction}
\label{sec:intro}

The \textit{alternating direction method of multipliers} (ADMM) is an invaluable element of the modern optimization toolbox.
ADMM decomposes complex optimization problems into sequences of simpler subproblems, often solvable in closed form; its simplicity, flexibility, and broad applicability, make ADMM a state-of-the-art solver in machine learning, signal processing, and many other areas \citep{boyd2011admm}.

%ADMM generally requires substantial user oversight making it difficult to adopt in commercial and black-box solvers.

It is well known that the efficiency of ADMM hinges on the careful selection of a \textit{penalty parameter}, which needs to be manually tuned by users for their particular problem instances.  In contrast, for gradient descent and proximal-gradient methods, adaptive (\textit{i.e.} automated) stepsize selection rules have been proposed, which essentially dispense with user oversight and dramatically boost performance \citep{barzilai1988two,fletcher2005barzilai,goldstein2014field,wright2009sparse,zhou2006gradient}.
% However, little work has been done to adopt adaptive stepsize rules for ADMM.

%To our knowledge, residual balancing~\cite{he2000alternating,boyd2011admm} is the only previous work that aims to improve practical performance for the general form of ADMM by adapting the penalty parameters for each iteration.

In this paper, we propose to automate and speed up ADMM by using stepsize selection rules adapted from the gradient descent literature, namely the Barzilai-Borwein  ``spectral'' method for smooth unconstrained problems \citep{barzilai1988two,fletcher2005barzilai}.   Since ADMM handles multi-term objectives and linear constraints, it is not immediately obvious how to adopt such rules.  The keystone of our approach is to analyze the dual of the ADMM problem, which can be written without constraints.  To ensure  reliability of the method, we develop a correlation criterion that safeguards it against inaccurate stepsize choices.  The resulting {\it adaptive ADMM} (AADMM) algorithm  is fully automated and  fairly insensitive to the initial stepsize, as testified for by a comprehensive set of experiments.

%\textbf{Notation} The $\ell_1$  and $\ell_2$ norms are denoted as $|\cdot|$ and $\| \cdot \|$, respectively, and  inner products as $\langle \cdot, \cdot \rangle$. The subdifferential of function $F$ at $x$ is denoted as $\partial F(x)$, and $F^*$ is its Fenchel conjugate, defined as  $F^*(y) = \sup_{x} \langle x, y\rangle - F(x)$ \cite{Rockafellar}.

\section{Background and Related Work}
\subsection{ADMM}
\label{sec:admm}
ADMM dates back to the 1970s \cite{gabay1976dual, glowinski1975approximation}. Its convergence was shown in the 1990s \cite{eckstein1992douglas}, and convergence rates have been the topic of much recent work, \textit{e.g.}, by  \citet{goldstein2014fast,he2015non,nishihara2015general}. In the last decade, ADMM became one of the tools of choice to handle a wide variety of optimization problems in machine learning, signal processing, and many other areas \citep{boyd2011admm}.

ADMM tackles problems in the form
\begin{equation}
\begin{split}
&\min_{u\in \bbR^n,v\in \bbR^m}  \hspace{0.7cm} H(u) + G(v),\\
&\quad \st~~  Au+Bv = b, \label{eq:prob}
\end{split}
\end{equation}
where $H:\bbR^n\rightarrow \bar{\bbR}$ and $G:\bbR^m\rightarrow \bar{\bbR}$ are closed, proper, convex functions, $A\in \bbR^{p \times n}$, $B\in \bbR^{p \times m}$, and $b \in \bbR^p$.  With $\gl \! \in\! \bbR^p $ denoting the dual variables (Lagrange multipliers), ADMM has the form
{\small
\begin{align}
%u_{k+1} =& \arg\min_{u} H(u) + \langle \gl_k, -Au \rangle + \frac{\tau_k}{2} \| b-Au-Bv_k\|_2^2
u_{k+1} =& \arg\min_{u} H(u) + \frac{\tau_k}{2} \| b-Au-Bv_k + \frac{\gl_k}{\tau_k}\|_2^2 \label{eq:updateu}\\
%v_{k+1} =& \arg\min_{v} G(v) + \langle \gl_k, -Bv \rangle + \frac{\tau_k}{2} \| b-Au_{k+1}-Bv\|_2^2 \label{eq:updatev}\\
v_{k+1} =& \arg\min_{v} G(v) + \frac{\tau_k}{2} \| b-Au_{k+1}-Bv + \frac{\gl_k}{\tau_k}\|_2^2 \label{eq:updatev}\\
\gl_{k+1} =& \gl_{k} +\tau_k (b - A u_{k+1} -B v_{k+1}), \label{eq:updatedual}
\end{align}
}%
where the sequence of \textit{penalties} $\tau_k$ is the only free choice, and has a high impact on the algorithm's speed.  Our goal is to automate this choice, by  adaptively tuning $\tau_k$ for optimal performance.

The convergence of the algorithm can be monitored using primal and dual ``residuals,'' both of which approach zero as the iterates become more accurate, and which are  defined as
\begin{eqnarray}
\begin{split}
& r_k = b-Au_k-Bv_k, \ \\ 
&\ d_k = \tau_k A^{T}B(v_k-v_{k-1}),
\end{split}
\end{eqnarray}
respectively \citep{boyd2011admm}.  The iteration is generally stopped when
\begin{equation}
\begin{split}
&\|r_k\|_2  \leq \ge^{tol} \max\{\|Au_k\|_2, \|Bv_k\|_2, \|b\|_2 \}  \\
&\|d_k\|_2 \leq \ge^{tol} \| A^{T} \gl_k\|_2, \label{eq:stop}
\end{split}
\end{equation}
where $\ge^{tol} > 0$ is the stopping tolerance.

\subsection{Parameter tuning and adaptation}
Relatively little work has been done on automating ADMM, \textit{i.e.}, on adaptively choosing $\tau_k$.  In the particular case of a strictly convex quadratic objective, criteria for choosing an optimal constant penalty have been recently proposed by \citet{ghadimi2015optimal,raghunathan2014alternating}. \citet{lin2011linearized} proposed a non-increasing sequence for the linearization parameter in ``linearized'' ADMM; however, they do not address the question of how to choose the penalty parameter in ADMM or its variants.

 %The optimal fixed penalty parameter is computed before the ADMM optimization based on the parameter of the given problems. Moreover, the computing of optimal parameter also depends on the formulation of steps of ADMM to solve the original problem.

Residual balancing (RB) \cite{he2000alternating,boyd2011admm} is the only available adaptive method for general form problems~\eqref{eq:prob}; it is based on the following observation: increasing $\tau_k$ strengthens the penalty term, yielding smaller primal residuals but larger dual ones; conversely, decreasing $\tau_k$ leads to larger primal and smaller dual residuals. As both residuals must be small at convergence, it makes sense to ``balance'' them, \textit{i.e.}, tune $\tau_k$ to keep both residuals of similar magnitude. A simple scheme for this goal is
\begin{equation}
\tau_{k+1} =
\begin{cases}
\eta \tau_k &~~\text{if}~~ \|r_k\|_2 > \mu \| d_k\|_2 \\
\tau_k/\eta &~~\text{if}~~ \|d_k\|_2 > \mu \| r_k\|_2\\
\tau_k &~~\text{otherwise},
\end{cases}\label{eq:RB}
\end{equation}
with $\mu > 1$ and $\eta > 1$ \cite{boyd2011admm}.  RB has recently been adapted to distributed optimization~\cite{song2015fast} and other primal-dual splitting methods \cite{goldstein2015adaptive}. ADMM with adaptive penalty is not guaranteed to converge, unless $\tau_k$ is fixed after a finite number of iterations~\cite{he2000alternating}.

Despite some practical success of the RB idea, it suffers from several flaws.  The relative size of the residuals depends on the scaling of the problem;  e.g., with the change of variable $u\gets 10 u$,  problem \eqref{eq:prob} can be re-scaled so that ADMM produces an equivalent sequence of iterates  with residuals of very different magnitudes. Consequently, RB criteria are arbitrary in some cases, and their performance varies wildly with different problem scalings (see Section \ref{sec:sensitivity}).  Furthermore, the penalty parameter may adapt slowly if the initial value is far from optimal. Finally, without a careful choice of $\eta$ and $\mu$, the algorithm may fail to converge unless adaptivity is turned off \cite{he2000alternating}.

\subsection{Dual interpretation of ADMM}
\label{sec:ADMM_DRS}
We now explain the close relationship between ADMM and \textit{Douglas-Rachdord splitting} (DRS) \cite{eckstein1992douglas,esser2009applications,goldstein2014fast}, which plays a central role in the proposed approach. The starting observation is that the dual of problem~\eqref{eq:prob} has the form
\begin{equation}
\min_{\zeta\in \bbR^p}  \underbrace{H^*(A^{T}\zeta) - \inprod{\zeta, b}}_{ \hat{H}(\zeta)} + \underbrace{G^*(B^{T}\zeta)}_{\hat{G}(\zeta)},
\label{eq:def_Hhat_Ghat}
\end{equation}
where $F^*$ denotes the Fenchel conjugate of $F$, defined as  $F^*(y) = \sup_{x} \langle x, y\rangle - F(x)$ \cite{Rockafellar}.

The DRS algorithm solves \eqref{eq:def_Hhat_Ghat} by generating two sequences $(\zeta_k)_{k\in\bbN} $ and $(\hat \zeta_k)_{k\in\bbN}$ according to
\begin{eqnarray}
0 & \in & \frac{\hat{\zeta}_{k+1}-\zeta_k}{\tau_k}  + \partial \hat H(\hat{\zeta}_{k+1}) +  \partial \hat G(\zeta_k) \label{eq:dr1}\\
0 & \in & \frac{{\zeta}_{k+1}-\zeta_k}{\tau_k} + \partial \hat H(\hat{\zeta }_{k+1}) + \partial \hat G(\zeta_{k+1}), \label{eq:dr2}
\end{eqnarray}
where we use the standard notation $\partial F(x)$ for the subdifferential of $F$ evaluated at $x$ \cite{Rockafellar}.

%To show the equivalence between ADMM and DRS, we observe that the dual of problem~\eqref{eq:prob} has the form \eqref{eq:dualform}
%where (see the derivation in \cite{esser2009applications,goldstein2014fast} )
%$$ \hat{H}(\gl) = H^*(A^{T}\gl) - \inprod{\gl, b}
%\quad\text{ and }  \quad \hat{G}(\gl) =  G^*(B^{T}\gl).$$

Referring back to ADMM in \eqref{eq:updateu}--\eqref{eq:updatedual}, and defining  $\hat{\gl}_{k+1} = \gl_{k} +\tau_k (b - A u_{k+1} -B v_{k}),$  the optimality condition for the minimization in \eqref{eq:updateu} is
\[
0   \in \partial H(u\kp) - A^T \underbrace{(\gl_k + \tau_k  ( b-Au\kp-Bv_k))}_{\hat{\gl}_{k+1}}
\]
which is equivalent to $A^T  \hat{\gl}_{k+1} \in \partial H(u\kp),$ thus\footnote{An important property relating $F$ and $F^*$ is that  $y \in \partial H(x) $ if and only if $ x\in \partial H^*(y)$ \cite{Rockafellar}.}  $u\kp \in \partial H^*(A^T  \hat{\gl}_{k+1}).$  A similar argument using the optimality condition for \eqref{eq:updatev} leads to $v_{k+1} \in  \partial G^*(B^{T}\gl_{k+1}).$ Recalling \eqref{eq:def_Hhat_Ghat}, we arrive at
{\small
\begin{align} \label{conjgrad}
 Au\kp-b \in  \partial \hat H( \hat{\gl}_{k+1}) \;\; \text{and} \;\; Bv_{k+1} \in \partial \hat{G}(\gl_{k+1}).
 \end{align}
 }%
 Using these identities, we finally have
    \begin{align}
    \hat \lambda\kp &=  \gl_{k} +\tau_k (b - A u_{k+1} -B v_{k}) \nonumber\\ & \in \lambda_k - \tau_k \bigl(  \partial \hat H( \hat{\gl}_{k+1}) + \partial \hat G( {\gl}_{k}) \bigr)
\\
     \lambda\kp &=  \gl_{k} +\tau_k (b - A u_{k+1} -B v_{k+1}) \nonumber\\ & \in  \lambda_k - \tau_k   \bigl(   \partial \hat H( \hat{\gl}_{k+1}) + \partial \hat G( {\gl}_{k+1}) \bigr),
     \end{align}
showing that the sequences $(\lambda_k)_{k\in\bbN} $ and $(\hat \lambda_k)_{k\in\bbN}$ satisfy  the same conditions  \eqref{eq:dr1} and \eqref{eq:dr2} as $(\zeta_k)_{k\in\bbN} $ and $(\hat \zeta_k)_{k\in\bbN}$, thus proving that ADMM  for problem~\eqref{eq:prob} is equivalent to DRS for its dual \eqref{eq:def_Hhat_Ghat}.

\subsection{Spectral  stepsize selection}
The classical gradient descent step for unconstrained minimization of a smooth function $F\! \! :\bbR^n\!\!\rightarrow\bbR$ has the form $x\kp = x_k - \tau_k \nabla F(x_k).$   
Spectral gradient methods, pioneered by Barzilai and Borwein (BB) \cite{barzilai1988two}, adaptively choose the stepsize $\tau_k$ to achieve fast convergence.

In a nutshell,  the standard (there are  variants) BB method  sets $\tau_k = 1/\alpha_k$, with $\alpha_k$ chosen such that   $\alpha_k I$ mimics the Hessian of $F$ over the last step, seeking a quasi-Newton step. A least squares criterion yields
{
\begin{equation}\small
\alpha_k = \arg\!\min_{\alpha\in \bbR} \|\nabla F(x_k) - \nabla F(x_{k-1}) - \alpha (x_k - x_{k-1})\|_2^2, \label{eq:def_BB}
\end{equation}
}%
which is an estimate of the curvature of $F$ across the previous step of the algorithm. BB gradient methods often dramatically outperform those with constant stepsize  \cite{fletcher2005barzilai,zhou2006gradient} and have been generalized to handle non-differentiable problems via proximal gradient methods \cite{wright2009sparse,goldstein2014field,goldstein2010high}. Finally, notice that \eqref{eq:def_BB} is equivalent to approximating the gradient $\nabla F(x_k)$ as  a linear function of $x_k$,
{
\begin{equation}\small
 \nabla F(x_k) \approx  \nabla F(x_{k-1}) + \alpha_k (x_k - x_{k-1}) = \alpha_k \, x_k + a_k, \label{eq:linear_grad}
\end{equation}
}%
where $a_k = \nabla F(x_{k-1}) -  \alpha_k \, x_{k-1} $. The observation that a local linear approximation of the gradient has an optimal parameter  equal to the inverse of the BB stepsize will play an important role below.

\section{Spectral penalty parameters}
 Inspired by the BB method, we propose a spectral  penalty parameter selection method for ADMM. We first derive a spectral stepsize rule for DRS,   and then adapt this rule to ADMM. Finally, we discuss safeguarding rules to prevent unexpected behavior when curvature estimates are inaccurate.
 % by exploiting the equivalence between ADMM and DRS (reviewed in Subsection~\ref{sec:ADMM_DRS}).  We also discuss safeguarding methods that prevent misbehavior of the rule when accurate curvature estimates cannot be obtained for the dual objective.

\subsection{Spectral stepsize for DRS}
\label{sec:dr2curv}
Consider the dual problem \eqref{eq:def_Hhat_Ghat}. Following the observation in \eqref{eq:linear_grad} about the BB method, we approximate $\partial\hat H$ and  $\partial \hat G$ at iteration $k$ as linear functions,
\begin{equation}
\partial \hat H(\hat \zeta) = \alpha_k \, \hat \zeta + \Psi_k ~~~~~\text{and}~~~~~\partial \hat G(\zeta) = \beta_k \, \zeta + \Phi_k, \label{eq:linearb2}
\end{equation}
where  $\ga_k > 0$, $\gb_k >0$ are local curvature estimates of dual functions $\hat{H}$ and $\hat{G}$, respectively, and $\Psi_k, \Phi_k \subset \bbR^p$. Once we obtain these curvature estimates, we will be able to exploit the following proposition.

\begin{proposition}[Spectral DRS]
  %When apply Douglas-Rachford splitting on two quadratic functions with diagonal Hessian of $\ga,\gb$, the optimal stepsize is $1/\sqrt{\ga\gb}$.
Suppose the DRS steps \eqref{eq:dr1}--\eqref{eq:dr2} are applied to  problem \eqref{eq:def_Hhat_Ghat}, where (omitting the subscript $k$ from $\alpha_k, \beta_k, \Psi_k, \Phi_k$ to lighten the notation in what follows)
\[
\partial \hat H(\hat \zeta) = \alpha \, \hat \zeta + \Psi ~~~~~\text{and}~~~~~\partial \hat G(\zeta) = \beta \, \zeta + \Phi . \label{eq:linearb}
\]
Then, the minimal residual of $\hat{H}(\zeta_{k+1}) + \hat{G}(\zeta_{k+1})$ is obtained by setting $\tau_k =1/\sqrt{\alpha\, \beta}$.
\label{rl:spectral}
\end{proposition}

\begin{proof}
Inserting \eqref{eq:linearb2} into the DRS step~\eqref{eq:dr1}--\eqref{eq:dr2} yields
\begin{align}
0 & \in  \frac{\hat{\zeta}_{k+1}-\zeta_k}{\tau}  + (\alpha \, \hat{\zeta}_{k+1} +  \Psi ) + (\beta\,  \zeta_k + \Phi), \label{eq:bbdr1}\\
0 & \in  \frac{{\zeta}_{k+1}-\zeta_k}{\tau} + (\alpha \, \hat{\zeta}_{k+1} +  \Psi ) + (\beta\,  \zeta_{k+1} + \Phi)\label{eq:bbdr2}.
\end{align}
From~\eqref{eq:bbdr1}--\eqref{eq:bbdr2}, we can explicitly get the update  for $\hat \zeta_{k+1}$ as
\begin{equation}
\hat{\zeta}_{k+1}  = \frac{1 - \beta \, \tau }{1+\alpha\, \tau} \zeta_{k} -\frac{a \tau + b \tau }{1+\alpha\, \tau} , \label{eq:tmp}
\end{equation}
where  $a \in \Psi$ and $b \in \Phi$, and for $\zeta_{k+1}$ as
\begin{align}
\zeta_{k+1} &= \frac{1}{1+\gb\, \tau} \zeta_{k} - \frac{\alpha\,\tau}{1+\beta\, \tau}\hat{\zeta}_{k+1} - \frac{a \, \tau + b \tau}{1+\beta\, \tau}  \\
&=  \frac{(1+\alpha\, \beta\, \tau^2)\zeta_{k} - (a + b) \tau}{(1+\alpha\, \tau)(1+\beta\, \tau)}, \label{eq:gl}
\end{align}
where the second equality results from  using the expression for $\hat{\zeta}_{k+1}$ in~\eqref{eq:tmp}.

The residual $r_{\mbox{\scriptsize DR}}$ at $\zeta_{k+1}$ is simply the magnitude of the subgradient
(corresponding to elements $a \in \Psi$ and $b \in \Phi$)
of the objective
%(corresponding to the elements $a_k \in \Psi_k$ and $b_k \in \Phi_k$),
that is given by
\begin{align}
\hspace{-2mm} r_{DR} & = \| (\alpha + \beta) \zeta_{k+1} + (a + b) \|_2\\
 &= \frac{1+\alpha\, \gb\, \tau^2}{(1+\alpha\, \tau)(1+\beta\, \tau)} \| (\ga+\gb) \zeta_{k} + (a + b)\|_2, \label{eq:drresiter}
\end{align}
where $\zeta_{k+1}$ in~\eqref{eq:drresiter} was substituted with~\eqref{eq:gl}. The optimal stepsize $\tau_{k}$ minimizes the residual
\begin{align}
\tau_{k} & = \arg\min_{\tau} r_{\mbox{\scriptsize DR}} =   \arg\max_{\tau} \frac{(1+\alpha\, \tau)(1+\beta\,\tau)}{1+\alpha\,\beta\,\tau^2} \\
&  =  \arg\max_{\tau} \frac{(\ga+\gb)\tau}{1+\ga\gb\tau^2}  \label{eq4}
 =  1/\sqrt{\ga\gb}  .
\end{align}
Finally (recovering the iteration subscript $k$), notice that $\tau_k = (\hat\alpha_k\, \hat \beta_k)^{1/2}$, where $\hat{\alpha}_k =1/\alpha_k$ and $\hat{\beta}_k = 1/\beta_k$ are the spectral gradient descent stepsizes for $\hat H$ and $\hat G$, at  $\hat\zeta_k$ and $\zeta_k$, respectively.
\end{proof}
%\tom{technically we haven't proved that we get minimal objective.  This proof should be modified to get that}

Proposition~\ref{rl:spectral} shows how to adaptively choose $\tau_k$: begin by obtaining linear estimates of the subgradients of the two terms in the dual objective \eqref{eq:def_Hhat_Ghat};  the geometric mean of these  optimal gradient descent stepsizes is then the optimal DRS stepsize, thus also the optimal ADMM penalty parameter, due to the equivalence shown in Subsection~\ref{sec:ADMM_DRS}.

\subsection{Spectral stepsize estimation}
\label{sec:curvestim}

We now address the estimation of $\hat \ga_k = 1/\alpha_k$ and $\hat \gb_k = 1/\beta_k$. These curvature parameters are estimated based on the results from iteration $k$ and an older iteration $k_0<k.$ Noting \eqref{conjgrad}, we define
\begin{equation} \label{eq:du}
\begin{split}
&\gD \hat{\gl}_{k} := \hat{\gl}_k - \hat{\gl}_{k_0}   \nonumber \\
 &\gD \hat H_{k} := \partial \hat H(\hat{\gl}_k) -  \partial \hat H(\hat{\gl}_{k_0}) = A(u_{k}-u_{k_0}).
 \end{split}
\end{equation}
Assuming, as above, a linear model for $\partial \hat H$, we expect $\gD \hat H_{k} \approx \alpha\,  \gD \hat{\gl}_{k} + a.$   As is typical in BB-type methods \cite{barzilai1988two,zhou2006gradient}, $\alpha$ is estimated via one of the two least squares problems
\begin{eqnarray*}
\min_{\alpha}\| \gD \hat H_{k} - \ga \gD \hat{\gl}_{k} \|_2^2 \,\  \text{or} \ \min_{\alpha}\|  \ga^{-1} \gD  \hat H_{k} -  \gD \hat{\gl}_{k} \|_2^2.
\end{eqnarray*}
The closed form solutions for the corresponding spectral stepsizes $\hga_k = 1/\ga_k$ are, respectively,
\begin{eqnarray}
\hga_k^{\mbox{\scriptsize SD}} = \frac{\inprod{\gD \hat{\gl}_{k}, \gD \hat{\gl}_{k}}}{\inprod{\gD \hat H_{k}, \gD \hat{\gl}_{k}}}
\,\ \text{and} \,\
\hga_k^{\mbox{\scriptsize MG}} = \frac{\inprod{\gD \hat H_{k}, \gD \hat{\gl}_{k}}}{\inprod{\gD \hat H_{k}, \gD \hat H_{k}}},
\end{eqnarray}
where, following  \citet{zhou2006gradient}, SD stands for {\em steepest descent} and MG for {\em minimum gradient}. The Cauchy-Schwarz inequality implies that $\hat{\ga}_k^{\mbox{\scriptsize SD}} \geq \hat{\ga}_k^{\mbox{\scriptsize MG}}.$  Rather than choosing one or the other, we suggest the hybrid stepsize rule proposed by \citet{zhou2006gradient},
% and \citet{goldstein2014field},
%defined as
\begin{eqnarray}
\hat{\ga}_k =
\begin{cases}
\hat{\ga}_k^{\mbox{\scriptsize MG}}&~~\text{if}~~2 \,\hat{\ga}_k^{\mbox{\scriptsize MG}} > \hat{\ga}_k^{\mbox{\scriptsize SD}} \\
\hat{\ga}_k^{\mbox{\scriptsize SD}} - \hat{\ga}_k^{\mbox{\scriptsize MG}} /2  &~~\text{otherwise.}
\end{cases}\label{eq:alpha}
\end{eqnarray}
The spectral stepsize $\hat{\gb}_k = 1/\gb_k $  is similarly set to
 \begin{align}
 \quad\hat{\gb}_k =
 \begin{cases}
 \hat{\gb}_k^{\mbox{\scriptsize MG}} & ~~\text{if}~~2\, \hat{\gb}_k^{\mbox{\scriptsize MG}} > \hat{\gb}_k^{\mbox{\scriptsize SD}}  \\
 \hat{\gb}_k^{\mbox{\scriptsize SD}} - \hat{\gb}_k^{\mbox{\scriptsize MG}}/2 & ~~\text{otherwise},
 \end{cases}\label{eq:beta}
 \end{align}
where ${\hat{\gb}_k^{\mbox{\scriptsize SD}}} = \inprod{\gD \gl_{k}, \gD \gl_{k}}/\inprod{\gD \hat G_{k}, \gD \gl_{k}}$,  $\hat{\gb}_k^{\mbox{\scriptsize MG}} = \inprod{\gD \hat G_{k}, \gD \gl_{k}}/\inprod{\gD \hat G_{k}, \gD \hat G_{k}}$, $ \gD \hat G_{k} = B(v_k-v_{k_0})$, and $\gD \gl_k = \gl_k -\gl_{k_0}$. It is important to note that  $\hat{\ga}_k $ and $\hat{\gb}_k$ are obtained from the iterates of ADMM, \textit{i.e.}, the user is not required to supply the dual problem.

\subsection{Safeguarding} %  : testing the quality of stepsize estimates}
On some iterations, the linear models (for one or both subgradients) underlying the spectral stepsize choice may be very inaccurate.  When this occurs, the least squares procedure may produce ineffective stepsizes.  The classical BB method for unconstrained problems uses a line search to safeguard against unstable stepsizes resulting from unreliable curvature estimates. In ADMM, however, there is no notion of ``stable'' stepsize (any constant stepsizes is stable),  thus line search methods are not applicable.  Rather, we propose to safeguard the method by assessing the quality of the curvature estimates, and only updating the stepsize if the curvature estimates satisfy a reliability criterion.

The linear model \eqref{eq:linearb2} assumes the change in dual (sub)gradient is linearly proportional to the change in the dual variables. To test the validity of this assumption, we measure the correlation between these quantities (equivalently, the cosine of their angle):
{
\begin{equation}\small
\ga^{\mbox{\scriptsize cor}}_k = \frac{\inprod{\gD \hat H_{k}, \gD \hat{\gl}_{k}}}{ \| \gD \hat H_{k}\| \, \| \gD \hat{\gl}_{k}\| } \,\ \text{and} \,\
\gb^{\mbox{\scriptsize cor}}_k = \frac{\inprod{\gD \hat G_{k}, \gD \gl_{k}}}{ \| \gD \hat G_{k}\| \, \| \gD \gl_{k}\| }. \label{eq:corr}
\end{equation}
}%
The spectral stepsizes are updated only if the correlations indicate the estimation is credible enough. The safeguarded spectral adaptive penalty rule is
{
\begin{align}\small
\tau_{k} =
\begin{cases}
\sqrt{\hat{\ga}_{k} \hat{\gb}_{k}} &\text{if}~ \ga^{\mbox{\scriptsize cor}}_k > \ge^{\mbox{\scriptsize cor}}~\text{and}~\gb^{\mbox{\scriptsize cor}}_k > \ge^{\mbox{\scriptsize cor}}\\
\hat{\ga}_{k} &\text{if}~ \ga^{\mbox{\scriptsize cor}}_k > \ge^{\mbox{\scriptsize cor}}~\text{and}~\gb^{\mbox{\scriptsize cor}}_k \leq \ge^{\mbox{\scriptsize cor}}\\
\hat{\gb}_{k} &\text{if}~ \ga^{\mbox{\scriptsize cor}}_k \leq \ge^{\mbox{\scriptsize cor}}~\text{and}~\gb^{\mbox{\scriptsize cor}}_k > \ge^{\mbox{\scriptsize cor}}\\
\tau_{k-1} &\text{otherwise},
\end{cases}~\label{eq:final}
\end{align}
}%
where $\ge^{\mbox{\scriptsize cor}}$ is a quality threshold for the curvature estimates, while $\hat{\ga}_{k}$ and $\hat{\gb}_k$ are the  stepsizes given by~\eqref{eq:alpha}--\eqref{eq:beta}. 
%We find $\gk^{\tau} = 0.2$ generally performs well.
 Notice that \eqref{eq:final} falls back to constant $\tau_k$ when both curvature estimates are deemed inaccurate.
 %Moreover, for non-differentiable problems, the curvatures often become hard to estimate when close to the optimal solution, which naturally stabilizes the penalty parameter for the convergence of ADMM.

\subsection{Adaptive ADMM}
Algorithm~\ref{alg} shows the complete \textit{adaptive ADMM} (AADMM). We suggest only updating the stepsize every $T_f$ iterations.
 Safeguarding threshold $\ge^{\mbox{\scriptsize cor}}=0.2$ and  $T_{f}=2$ generally perform well. 
The overhead of AADMM over ADMM is modest: only a few inner products plus the storage to keep one previous iterate.

\begin{algorithm}
	\caption{
		Adaptive ADMM (AADMM)
	}
	\label{alg}
	\begin{algorithmic}
		\REQUIRE  initialize $v_0$, $\gl_0$, $\tau_0$, $k_0=0$
		\WHILE{not converge by \eqref{eq:stop} \textbf{and} $k <\text{maxiter}$}
		\STATE {\small $u_{k+1} = \arg\min_{u} H(u) + \frac{\tau_k}{2} \| b-Au-Bv_k + \frac{\gl_k}{\tau_k}\|_2^2 $}
		\STATE {\small $v_{k+1} = \arg\min_{v} G(v) + \frac{\tau_k}{2} \| b-Au_{k+1}-Bv + \frac{\gl_k}{\tau_k}\|_2^2 $}
		\STATE {\small $\gl_{k+1} \gets \gl_{k} +\tau_k (b - A u_{k+1} -B v_{k+1})$}
		\IF{$\text{mod}(k, T_{f}) = 1$}
		\STATE $\hgl_{k+1} = \gl_{k} +\tau_k (b - A u_{k+1} -B v_{k})$
		\STATE {\small Estimate spectral stepsizes $\hga_{k+1},\hgb_{k+1}$ in (\ref{eq:alpha}, \ref{eq:beta})}
		\STATE {\small Estimate correlations $\ga_{k+1}^{\mbox{\scriptsize cor}},\gb_{k+1}^{\mbox{\scriptsize cor}}$ in \eqref{eq:corr}}
		\STATE {\small Update $\tau_{k+1}$ in \eqref{eq:final}}
		\STATE  $k_0 \gets k$
		\ELSE
		\STATE $\tau_{k+1} \gets \tau_{k}$
		\ENDIF
		\STATE $k \gets k+1$
		\ENDWHILE
	\end{algorithmic}
\end{algorithm}

\subsection{Convergence}
\label{sec:convgr}
\citet{he2000alternating} proved that convergence is guaranteed for ADMM with adaptive penalty when either of the two following conditions are satisfied:
\begin{condition}[Bounded increasing]\label{as1}
	{\small
		\begin{align}
			\sum_{k=1}^{\infty} (\eta\itk)^2 < \infty,
			\ \text{where} \ \eta\itk =
			\sqrt{\max\{\frac{\tau\itk}{\tau\km}, \, 1\}-1}.
		\end{align}
	}%
\end{condition}

\begin{condition}[Bounded decreasing]\label{as2}
	{\small
		\begin{align}
			\sum_{k=1}^{\infty} (\gt\itk)^2 < \infty,
			\ \text{where} \ \gt\itk =
			\sqrt{\max\{\frac{\tau\km}{\tau\itk}, \, 1\}-1}.
		\end{align}
	}%
\end{condition}
\cref{as1} (\cref{as2}) suggests that increasing (decreasing) of adaptive penalty is bounded. In practice, these conditions can be satisfied by turning off adaptivity after a finite number of steps, which we have found unnecessary  in our experiments with AADMM.

\section{Experiments}

\begin{table*}[t]
\centering
\caption{\small Iterations (and runtime in seconds) for the various algorithms and applications described in the text. Absence of convergence after $n$ iterations is indicated as $n+$. AADMM is the proposed Algorithm~\ref{alg}.
}
\setlength{\tabcolsep}{3pt}
\small
\begin{threeparttable}
\begin{tabular}{c|c|c||c|c|c|>{\bfseries}c}
\hline
Application & Dataset & \tabincell{c}{\#samples $\times$ \\ \#features\tnote{1}} & \tabincell{c}{Vanilla\\ ADMM} & \tabincell{c}{Fast \\ADMM} & \tabincell{c}{Residual\\ balance} & \tabincell{c}{Adaptive \\ADMM}\\
\hline\hline
\multirow{6}{*}{\tabincell{c}{Elastic net\\regression}} & Synthetic & 50 $\times$ 40 & 2000+ (1.64) & 263 (.270)  & 111 (.129) & 43 (.046) \\
& Boston & 506 $\times$ 13 & 2000+ (2.19) & 208 (.106) & 54 (.023) & 17 (.011) \\
& Diabetes & 768 $\times$ 8 & 594 (.269) & 947 (.848) & 28 (.020) & 10 (.005) \\
& Leukemia & 38 $\times$ 7129 & 2000+ (22.9) & 2000+ (24.2) & 1737 (19.3) & 152 (1.70) \\
& Prostate & 97 $\times$ 8 & 548 (.293) & 139 (.049) & 29 (.015) & 16 (.012) \\
& Servo & 130 $\times$ 4 & 142 (.040) & 44 (.017) & 27 (.012) & 13 (.007) \\
\hline
\multirow{4}{*}{\tabincell{c}{Low rank \\least squares}} & Synthetic & 1000 $\times$ 200 & 543(31.3) & 129(7.30) & 75(5.59) & 13(.775) \\
& Madelon & 2000 $\times$ 500 & 1943(925) & 193(89.6) & 133(60.9) & 27(12.8) \\
& Sonar & 208 $\times$ 60 & 1933(9.12) & 313(1.51) & 102(.466) & 31(.160) \\
& Splice & 1000 $\times$ 60 & 1704(38.2) & 189(4.25) & 92(2.04) & 18(.413) \\
\hline
\multirow{4}{*}{\tabincell{c}{QP and \\dual SVM}} & Synthetic & 250 $\times$ 500 & 439 (6.15) & 535 (7.8380) & 232 (3.27) & 71 (.984) \\
%& Ionosphere & 351 $\times$ 34 & 1585 (6.01) & 1917 (8.57) & 209 (.816) & 104 (.422) \\
& Madelon & 2000 $\times$ 500 & 100 (14.0) & 57 (8.14) & 28 (4.12) & 19 (2.64) \\
& Sonar & 208 $\times$ 60 & 139 (.227) & 43 (.075) & 37 (.069) & 28 (.050) \\
& Splice & 1000 $\times$ 60 & 149 (4.9) & 47 (1.44) & 39 (1.27) & 20 (.681) \\
\hline
\multirow{4}{*}{\tabincell{c}{Basis \\pursuit}} & Synthetic &  10 $\times$ 30   & 163 (.027) & 2000+ (.310) & 159 (.031) & 114 (.026) \\
& Human1 & 1024 $\times$ 1087 & 2000+ (2.35) & 2000+ (2.41) & 839 (.990) & 503 (.626) \\
& Human2 & 1024 $\times$ 1087 & 2000+ (2.26) & 2000+ (2.42) & 875 (1.03) & 448 (.554) \\
& Human3 & 1024 $\times$ 1087 & 2000+ (2.29) & 2000+ (2.44) & 713 (.855) & 523 (.641) \\
\hline
\multirow{7}{*}{\tabincell{c}{Consensus \\ logistic \\regression}} & Synthetic & 1000 $\times$ 25 & 301 (3.36) & 444 (3.54) & 43 (.583) & 22 (.282) \\
& Madelon & 2000 $\times$ 500 & 2000+ (205) & 2000+ (166) & 115 (42.1) & 23 (20.8) \\
& Sonar & 208 $\times$ 60 & 2000+ (33.5) & 2000+ (47) & 106 (2.82) & 90 (1.64) \\
& Splice & 1000 $\times$ 60 & 2000+ (29.1) & 2000+ (43.7) & 86 (1.91) & 22 (.638) \\
& News20 & 19996 $\times$ 1355191 & 69 (5.91e3) & 32 (3.45e3) & 18 (1.52e3) & 16 (1.2e3) \\
& Rcv1 & 20242 $\times$ 47236 & 38 (177) & 23 (122) & \textbf{13 (53.0)} & 12 (53.9) \\
& Realsim & 72309 $\times$ 20958 & 1000+ (2.73e3) & 1000+ (1.86e3) & 121 (558) & 22 (118) \\
\hline
\multirow{6}{*}{\tabincell{c}{Semidefinite \\ programming}}
& hamming-7-5-6 & 128 $\times$ 1792 & 455(1.78) & 2000+(8.60) & 1093(4.21) & 284(1.11) \\
& hamming-8-3-4 & 256 $\times$ 16128 & 418(6.38) & 2000+(29.1) & 1071(16.5) & 118(2.02) \\
& hamming-9-5-6 & 512 $\times$ 53760 & 2000+(187) & 2000+(187) & 1444(131) & 481(53.1) \\
& hamming-9-8 & 512 $\times$ 2304 & 2000+(162) & 2000+(159) & 1247(97.2) & 594(52.7) \\
& hamming-10-2 & 1024 $\times$ 23040 & 2000+(936) & 2000+(924) & 1194(556) & 391(193) \\
& hamming-11-2 & 2048 $\times$ 56320 & 2000+(6.43e3) & 2000+(6.30e3) & 1203(4.15e3)  & 447(1.49e3) \\
\hline
\end{tabular}%
\label{tab:exp}%
\begin{tablenotes}
    \item[1] \#constrains $\times$ \#unknowns for canonical QP; \#vertices $\times$ \#edges for SDP.
\end{tablenotes}
\end{threeparttable}
\end{table*}%

\subsection{Experimental setting}
We consider several applications to demonstrate the effectiveness of the proposed AADMM. We focus on statistical  problems involving non-differentiable objectives:
linear regression with elastic net regularization~\cite{efron2004least,goldstein2014fast}, low rank least squares~\cite{yang2013linearized,xu2015bmvc}, quadratic programming (QP)~\cite{boyd2011admm,ghadimi2015optimal,goldstein2014fast,raghunathan2014alternating}, basis pursuit~\cite{boyd2011admm,goldstein2014fast},  consensus $\ell_1$-regularized logistic regression~\cite{boyd2011admm}, and semidefinite programming~\cite{burer2003nonlinear,wen2010alternating}. We use both synthetic and benchmark datasets (obtained from the UCI repository and the LIBSVM page) used  by~\citet{efron2004least,lee2006efficient,liu2009large,schmidt2007fast,wright2009sparse}, and \citet{zou2005regularization}.
For the small and medium sized datasets, the features are  standardized to zero mean and unit variance, whereas for the large and sparse datasets the features are scaled to be in $[-1,\, 1]$.

For comparison, we implemented \textit{vanilla} ADMM (fixed stepsize), fast ADMM with a restart strategy~\cite{goldstein2014fast}, and  ADMM with residual balancing~\cite{boyd2011admm,he2000alternating}, using \eqref{eq:RB} with $\mu = 10$ and $\eta=2$, and adaptivity was turned off after 1000 iterations to guarantee convergence. The proposed AADMM is implemented as shown in Algorithm~\ref{alg}, with fixed parameters $\ge^{\mbox{\scriptsize cor}}=0.2$ and $T_{f}=2$.

We set the stopping tolerance to $\ge^{tol} = 10^{-5}, 10^{-3},$ and $0.05$ for small, medium, and large scale problems, respectively. The initial penalty $\tau_0=0.1$ is used for all problems, except the canonical QP, where $\tau_0$ is set to the value proposed for quadratic problems by~\citet{raghunathan2014alternating}.
For each problem, the same randomly generated  initial variables  $v_0, \gl_0$ are used for ADMM and all the variants thereof.

\subsection{Applications}
\label{sec:application}
\textbf{Elastic net (EN)} is a modification of $\ell_1$-regularized linear regression (a.k.a. LASSO) that helps preserve groups of highly correlated variables~\cite{zou2005regularization,goldstein2014fast} and requires solving
\begin{eqnarray}
\min_x \frac{1}{2} \| Dx - c \|_2^2 + \rho_1 \|x\|_1 + \frac{\rho_2}{2} \| x\|_2^2, \label{eq:enet}
\end{eqnarray}
where, as usual, $\|\cdot\|_1 $ and $\|\cdot \|_2$ denote the $\ell_1$ and $\ell_2$ norms, $D$ is a data matrix, $c$ contains measurements, and $x$ is the vector of regression coefficients. One way to apply ADMM to this problem is to rewrite it as
\begin{equation}
\begin{split}
&\min_{u,v} \frac{1}{2} \| Du - c \|_2^2 + \rho_1 \|v\|_1 + \frac{\rho_2}{2} \| v\|_2^2
\\
&\st~~u-v=0.
\end{split}
\end{equation}
The synthetic dataset introduced by~\citet{zou2005regularization} and realistic dataset introduced by~\citet{efron2004least,zou2005regularization} are investigated. Typical parameters $\rho_1 = \rho_2=1$ are used in all experiments.

\textbf{Low rank least squares (LRLS)} uses the nuclear matrix norm (sum of singular values) as the convex surrogate of matrix rank,
\begin{eqnarray}
\min_{X} \frac{1}{2} \|DX - C\|_F^2 + \rho_1 \| X \|_* + \frac{\rho_2}{2} \| X\|_F^2, \label{LRLS}
\end{eqnarray}
where $\|\cdot\|_* $ denotes the nuclear norm, $\|\cdot\|_F $ is the Frobenius norm, $D\in \bbR^{n \times m}$ is a data matrix, $C\in \bbR^{n \times d}$ contains measurements, and $X\in \bbR^{m \times d}$ is the variable matrix.  ADMM can be applied after rewriting \eqref{LRLS} as \citet{yang2013linearized,xu2015bmvc}
\begin{equation}
\begin{split}
& \min_{U, V} \frac{1}{2} \|DU - C\|_F^2 + \rho_1 \| V \|_* + \frac{\rho_2}{2} \| V\|_F^2,\\
& \st~~U-V=0.
\end{split}
\end{equation}
A synthetic problem is constructed using a random data matrix $D \in \bbR^{1000\times 200}$, a low rank matrix $X \in \bbR^{200 \times 500}$, and $C = DW + \text{Noise}$. We use the binary classification problems introduced by~\citet{lee2006efficient} and \citet{schmidt2007fast}, where  each column of $X$ represents a linear exemplar classifier, trained with a positive sample and all negative samples~\citep{xu2015bmvc}; $\rho_1 = \rho_2 = 1$ is used for all experiments.

\textbf{Support vector machine (SVM) and QP:} the dual  of the SVM learning problem is a QP
\begin{equation}
\begin{split}
&\min_{z} \, \frac{1}{2} z^{T}Qz - e^{T}z \\
&\st~~ c^{T}z = 0 ~~\mbox{and} ~~ 0 \leq z \leq C,
\end{split}
\end{equation}
where $z$ is the SVM dual variable, $Q$ is the kernel matrix, $c$ is a vector of labels,  $e$ is a vector of ones, and  $C > 0$~\cite{chang2011libsvm}. We also consider the canonical QP
\begin{eqnarray}
\min_x \, \frac{1}{2}x^{T}Qx + q^{T}x~~~~\st~~Dx \leq c, \label{eq:qp}
\end{eqnarray}
which can be solved by applying ADMM to
\begin{equation}
\begin{split}
&\min_{u,v} \, \frac{1}{2}u^{T}Qu + q^{T} u + \iota_{ \{z:\, z_i \leq c\} } (v)\\
&\st~~Du - v =0 ; \label{eq:qpadmm}
\end{split}
\end{equation}
here, $\iota_{S}$ is the indicator function of set $S$: $\iota_{S}(v) = 0$, if $v\in S$, and $\iota_{S}(v) = \infty$, otherwise.

We study classification problems from~\citet{lee2006efficient} and \citet{schmidt2007fast} with $C=1$, and a random synthetic QP \citep{goldstein2014fast}, where $Q\in\bbR^{500 \times 500}$ with condition number $\simeq 4.5\times 10^5$.

\textbf{Basis pursuit (BP)} seeks a sparse representation of a vector $c$ by solving the constrained  problem
\begin{eqnarray}
\min_{x} \| x \|_1 ~~~~\st~~ Dx=c, \label{eq:basis_pursuit}
\end{eqnarray}
where $D \in \bbR^{m \times n}\!, c\in \bbR^m\!, m<n$.
An extended form with $\hat{D} = [D, I] \in \bbR^{m \times (n+m)}$ has been used to reconstruct occluded and corrupted faces~\cite{wright2009robustface}.
To apply ADMM,  problem \eqref{eq:basis_pursuit} is rewritten as
\begin{eqnarray}
\min_{u,v} \iota_{\{z:\, Dz=c\} }(u) + \| v \|_1~~~~\st~~u - v =0.
\end{eqnarray}
We experiment with synthetic random $D \in \bbR^{10 \times 30}.$  We also use a data matrix for face reconstruction from the Extended Yale B Face dataset~\cite{wright2009sparse}, where each frontal face image is scaled to $32 \times 32$. For each human subject, an image is selected and corrupted with $5\%$ noisy pixels, and the remaining  images from the same subject are used to reconstruct the corrupted image.

\textbf{Consensus $\ell_1$-regularized logistic regression} is formulated as a distribute optimization problem with the form
\begin{equation}
\begin{split}
&\min_{x_i, z} \sum_{i=1}^{N} \sum_{j=1}^{n_i} \log(1+\exp(-c_jD^T_jx_i)) + \rho \|z\|_1 \\
&\st~~ x_i - z = 0, i=1,\ldots,N,
\end{split}
\end{equation}
where $x_i \in \bbR^m$ represents the local variable on the $i$th distributed node, $z$ is the global variable, $n_i$ is the number of samples in the $i$th block, $D_j \in \bbR^m$ is the $j$th sample, and $c_j\in \{-1,1\}$ is the corresponding label. The goal of this example is to test AADMM also in distributed/consensus problems, for which ADMM has become an important tool~\cite{boyd2011admm}. 
%We apply homogeneous coordinates to absorb the bias term into variable $x_i$.

A  synthetic problem is constructed with Gaussian random data and sparse ground truth solutions. 
Binary classification problems from~\citet{lee2006efficient,liu2009large}, and \citet{schmidt2007fast} are also used to test the effectiveness of the proposed method. We use
 $\rho=1$, for small and medium datasets, and $\rho=5$  for the large datasets to encourage sparsity. We  split the data equally into two blocks and use a loop to simulate the distributed computing of consensus subproblems.

\textbf{Semidefinite programming (SDP)} solves the problem
\begin{eqnarray}
\min_X \inprod{F, X} ~~\st X \succeq 0, ~~\mcD(X) = c,  \label{eq:sdp}
\end{eqnarray}
where $X \!\!\! \succeq \!\!\! 0$ means that $X$ is positive semidefinite, $X, \, F, \, D_i \in \bbR^{n \times n}$ are symmetric matrices, inner product $\inprod{X, Y} = \text{trace}(X^T Y)$, and $\mcD(X) = (\inprod{D_1, X}, \ldots, \inprod{D_m, X})^T$. ADMM is applied to the dual form of \eqref{eq:sdp},
\begin{eqnarray}
\min_{y, S} ~-c^T y  ~~\st \mcD^*(y) + S = F,~~S \succeq 0, \label{eq:sdp_dual}
\end{eqnarray}
where $\mcD^*(y) = \sum_{i=1}^{m} y_i D_i$, and $S$ is a symmetric positive semidefinite matrix.

As test data, we use 6 graphs from the \textit{Seventh DIMACS Implementation Challenge on Semidefinite and Related Optimization Problems} (following   \citet{burer2003nonlinear}). 

\begin{figure}[t]
\centerline{
\includegraphics[width=0.8\linewidth]{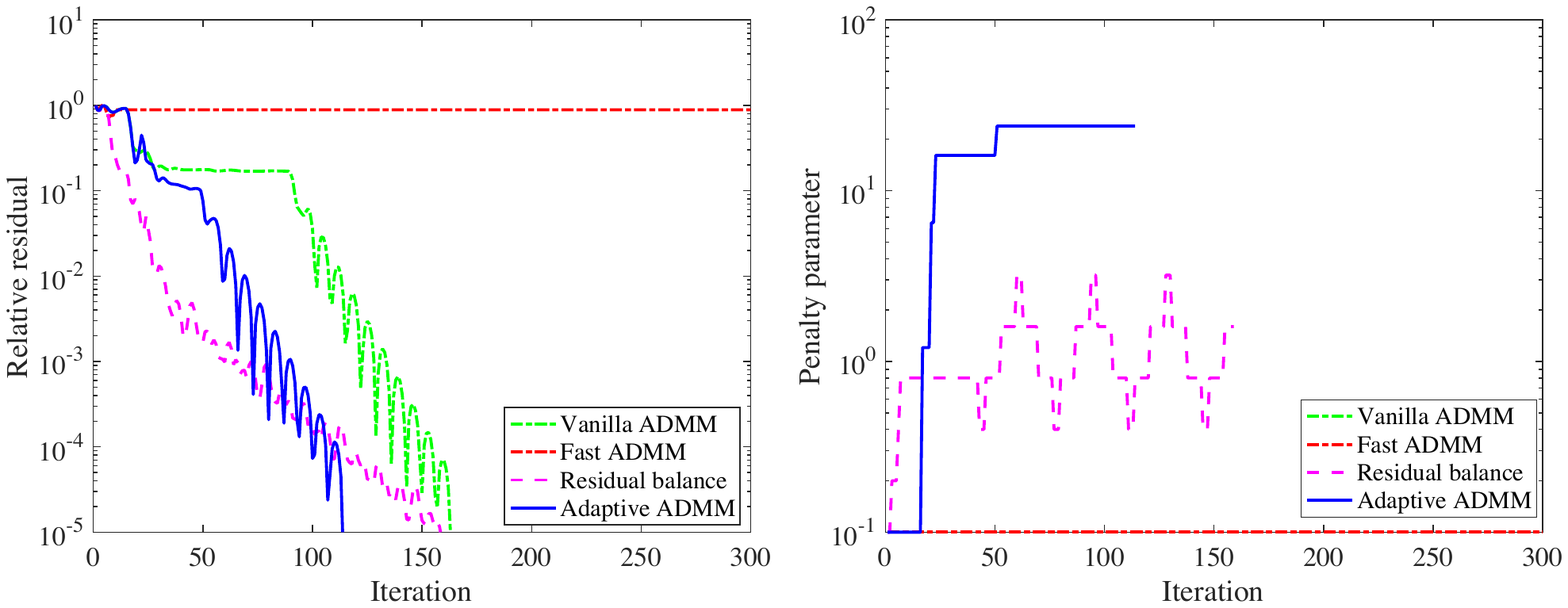}
}
\centerline{
\includegraphics[width=0.8\linewidth]{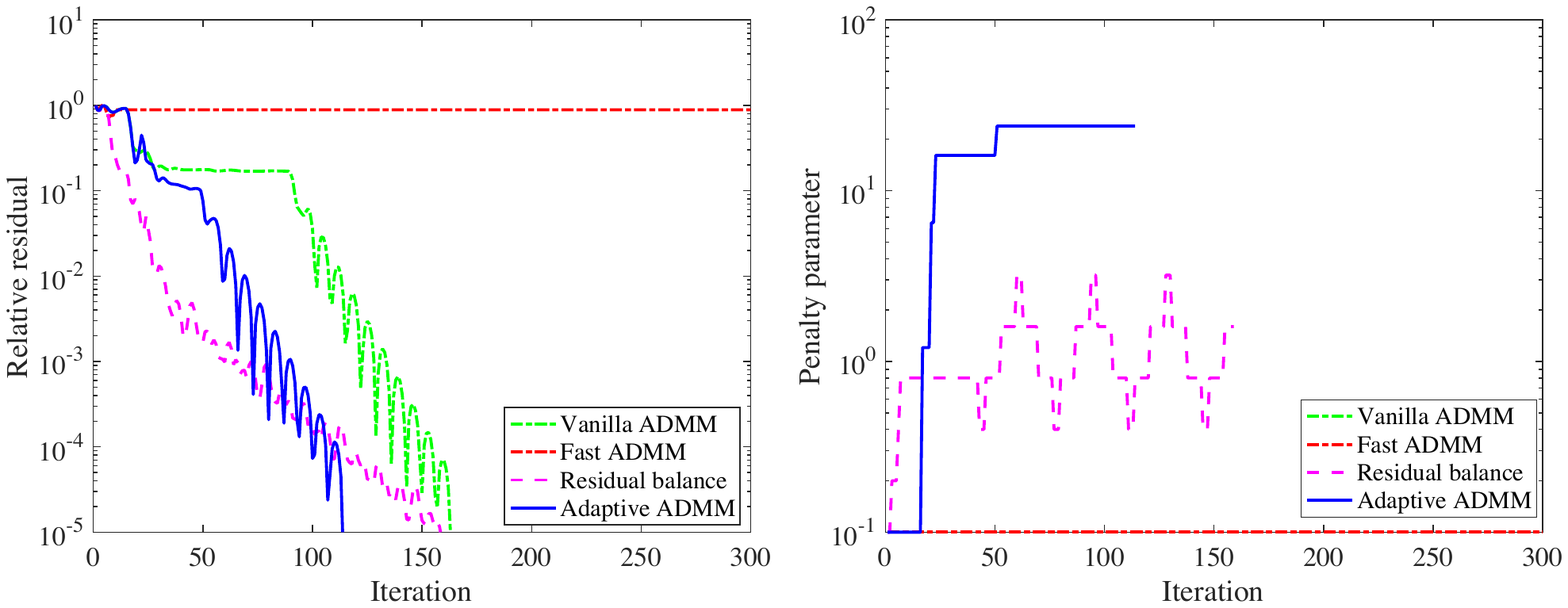}
}
\caption{\small Relative residual (top) and penalty parameter (bottom) for the synthetic basis pursuit (BP) problem.}
\label{fig:res}
\end{figure}

\subsection{Convergence results}

Table~\ref{tab:exp} reports the convergence speed of ADMM and its variants for the applications described in Subsection~\ref{sec:application}.
Vanilla ADMM with fixed stepsize does poorly in practice: in 13 out of 23 realistic datasets, it fails to converge  in the maximum number of iterations.
%~\footnote{2000 for small and medium datasets and 1000 for large datasets.}.
Fast ADMM~\cite{goldstein2014fast} often outperforms vanilla ADMM, but does not compete with the proposed AADMM,  which  also  outperforms residual balancing in all test cases except in the Rcv1 problem for consensus logistic regression.

\cref{fig:res} presents the relative residual (top) and penalty parameter (bottom) for the synthetic BP problem. The relative residual is defined as 
\[
\max \left\{ \frac{\|r_k\|_2}{\max\{\|Au_k\|_2, \|Bv_k\|_2, \|b\|_2 \} }, \frac{\|d_k\|_2 }{ \| A^{T} \gl_k\|_2 } \right\},
\]
 which is  based on stopping criterion \eqref{eq:stop}. Fast ADMM often restarts and is slow to converge. The penalty parameter chosen by RB oscillates. AADMM quickly adapts the penalty parameter and converges fastest.

%\vspace{0.5cm}
\begin{figure*}[hbtp]
\centerline{
\includegraphics[width=\linewidth]{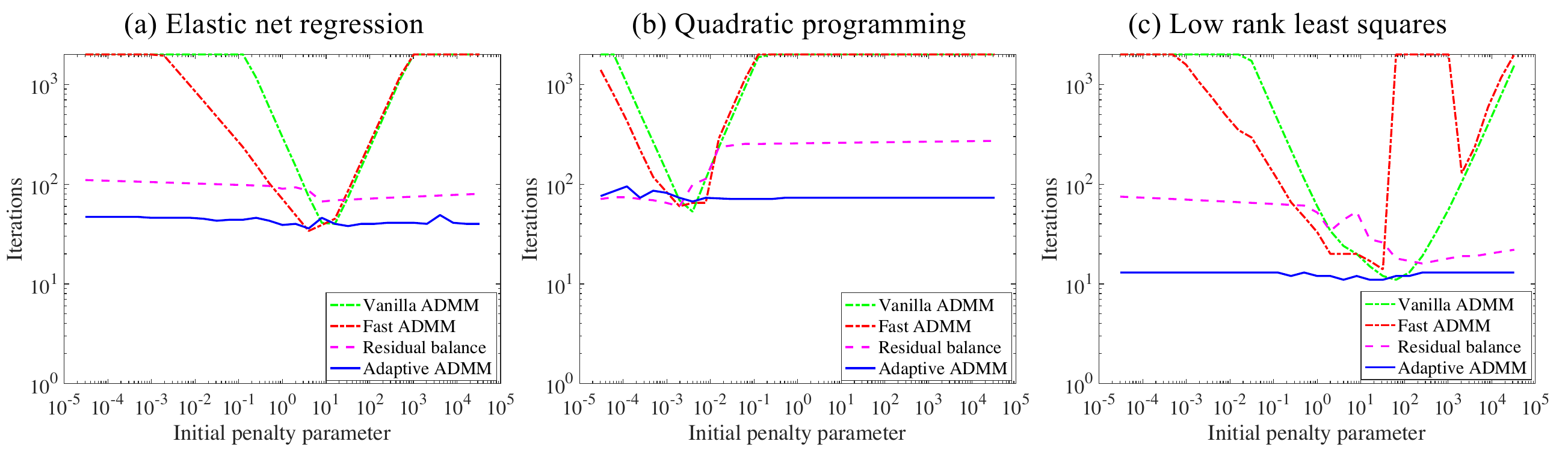}
}
\vspace{0.3cm}
\centerline{
\includegraphics[width=\linewidth]{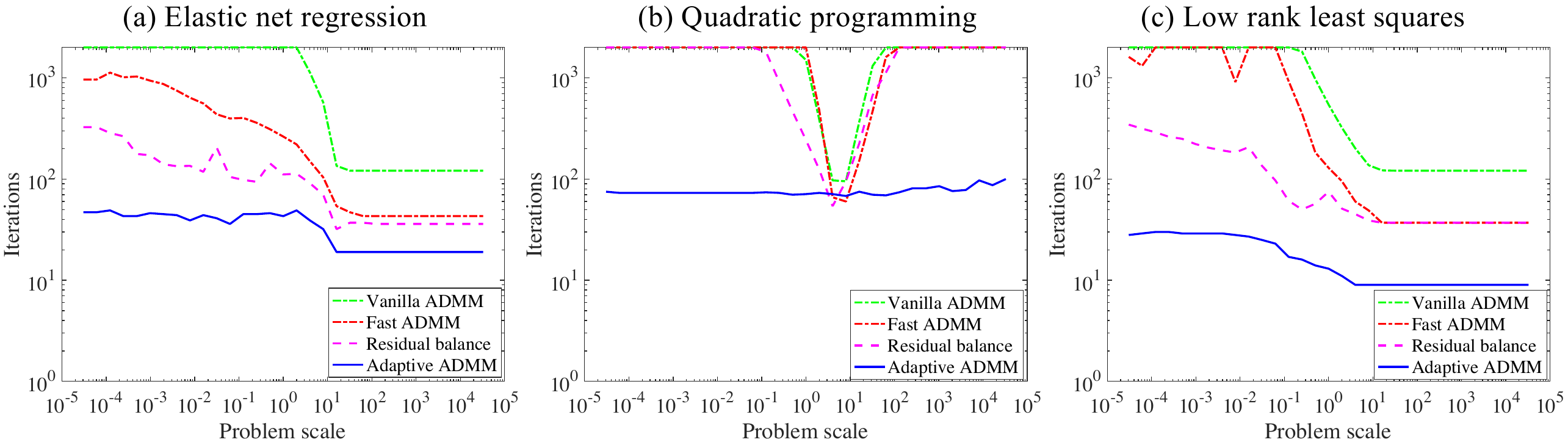}
}
\vspace{-0.2cm}
\caption{\small Top row: sensitivity of convergence speed to initial penalty parameter $\tau_0$ for EN, QP, and LRLS. Bottom row: sensitivity to problem scaling $s$ for EN, QP, and LRLS. }
\label{fig:tau}
\end{figure*}

\subsection{Sensitivity
%to initial stepsize and problem scaling
} \label{sec:sensitivity}
We study the sensitivity of the different ADMM variants to problem scaling and  initial penalty parameter ($\tau_0$). Scaling sensitivity experiments were done by multiplying  the measurement vector $c$ by a scalar $s$. Fig.~\ref{fig:tau}  presents iteration counts for a wide range of values of initial penalty $\tau_0$ (top) and problems scale $s$ (bottom) for EN regression, canonical QP, and LRLS with synthetic datasets.  Fast  ADMM and vanilla ADMM use the fixed initial penalty parameter $\tau_0$, and are highly sensitive to this choice, as shown in Fig.~\ref{fig:tau}; in contrast, AADMM  is very stable with respect to $\tau_0$ and the scale $s$.

Finally, \cref{fig:corr} presents iteration counts when applying AADMM with various safeguarding correlation thresholds $\ge^{{\scriptsize \text{cor}}}.$  When $\ge^{{\scriptsize \text{cor}}} =0$ the new penalty value is always accepted, and when $\ge^{{\scriptsize \text{cor}}} \! =\! 1$ the penalty parameter is never changed.  The proposed AADMM method is insensitive to $\ge^{{\scriptsize \text{cor}}}$ and performs well for a wide range of $\ge^{{\scriptsize \text{cor}}} \in [0.1, \, 0.4]$ for various applications.

\begin{figure}[t]
\vspace{-0.1cm}
\centerline{
\includegraphics[width=0.8\linewidth]{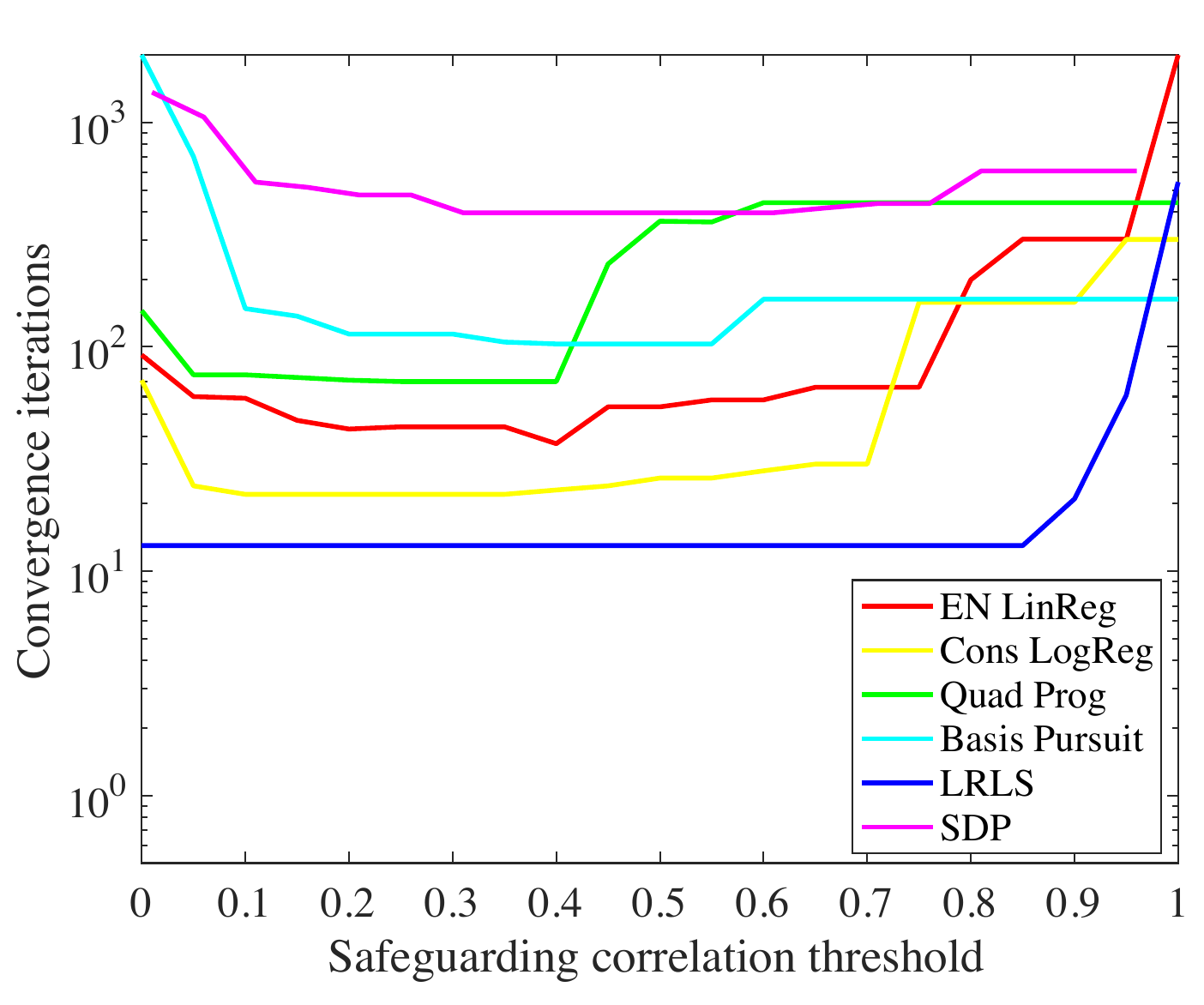}
\vspace{-1mm}}
\caption{\small Sensitivity of convergence speed to safeguarding threshold $\ge^{{\scriptsize \text{cor}}}$ for proposed AADMM. Synthetic problems of various applications are studied. Best viewed in color.}
\label{fig:corr}
\vspace{0.2cm}
\end{figure}

%\vspace{-0.5cm}
\section{Conclusion}
We have proposed \textit{adaptive ADMM} (AADMM), a new variant of the popular ADMM algorithm that tackles one of its fundamental drawbacks:  critical dependence on a penalty parameter that needs careful tuning. This drawback has made ADMM difficult to use by non-experts, thus AADMM has the potential to contribute to wider and easier applicability of this highly flexible and efficient optimization tool.  Our approach imports and adapts the Barzilai-Borwein ``spectral'' stepsize method from the smooth optimization literature, tailoring it to the more general class of problems handled by ADMM.  The cornerstone of our approach is the fact that ADMM is equivalent to Douglas-Rachford splitting (DRS) applied to the dual problem, for which we develop a spectral stepsize selection rule; this rule is then translated into a criterion to select the penalty parameter of ADMM. A safeguarding function that avoids unreliable stepsize choices finally yields  AADMM. Experiments on a comprehensive range of problems and datasets have shown that AADMM outperforms other variants of ADMM  and is robust with respect to initial parameter choice and problem scaling.

\subsubsection*{Acknowledgments}
TG and ZX were supported by the US Office of Naval Research (N00014-17-1-2078), and by the US National Science Foundation (CCF-1535902). MF was partially supported by the Funda\c{c}\~{a}o para a
Ci\^{e}ncia e Tecnologia, grant UID/EEA/5008/2013. 

%\clearpage

%\small
%\bibliographystyle{abbrvnat}
%\bibliographystyle{abbrv}
\bibliographystyle{plainnat}
\bibliography{admm}
\end{document}